\documentclass{article}
\usepackage{ijcai20}

\usepackage{times}
\usepackage{soul}
\usepackage{url}
\usepackage[hidelinks]{hyperref}
\usepackage[utf8]{inputenc}
\usepackage[small]{caption}
\usepackage{graphicx}
\usepackage{amsmath}
\usepackage{amsthm}
\usepackage{booktabs}
\usepackage{algorithmic}
\usepackage{ifthen}

\def\full{true}
\newcommand{\fullversion}[2]{\ifthenelse{\equal{\full}{true}}
	{#1}
	{#2}
}

\title{Greedy Convex Ensemble}

\author{
Thanh Tan Nguyen$^1$
\and
Nan Ye$^2$
\And
Peter Bartlett$^{3}$
\affiliations
$^1$Queensland University of Technology\\
$^2$The University of Queensland \\
$^3$UC Berkeley
\emails
tan1889@gmail.com,
nan.ye@uq.edu.au,
bartlett@cs.berkeley.edu
}

\usepackage[T1]{fontenc}    
\usepackage{amsfonts}       
\usepackage{nicefrac}       
\usepackage{microtype}      

\usepackage{amssymb}
\usepackage{mathtools}
\usepackage{thmtools,thm-restate}
\usepackage[linesnumbered,ruled]{algorithm2e}
\SetKwRepeat{Do}{do}{while}%
\SetAlgoNoEnd%
\usepackage{cleveref}
\crefname{equation}{Eq.}{Eqs.}
\crefname{figure}{Fig.}{Figs.}

\newcommand{\E}{{\mathbb E}}
\newtheorem{lem}{Lemma}

\newtheorem{prop}{Proposition}
\newtheorem{thm}{Theorem}
\newcommand{\setF}{{\mathcal F}}
\newcommand{\setG}{{\mathcal G}}
\newcommand{\setH}{{\mathcal H}}

\newcommand{\setR}{{\mathcal R}}
\newcommand{\setT}{{\mathcal T}}
\newcommand{\setX}{{\mathcal X}}
\newcommand{\setY}{{\mathcal Y}}
\newcommand{\setZ}{{\mathcal Z}}
\newcommand{\bx}{{\mathbf x}}

\newcommand{\bz}{{\mathbf z}}
\newcommand{\I}{{\mathbb I}}
\newcommand{\realnum}{{\mathbf R}}

\newcommand{\argmin}{{\arg\min}}

\newcommand{\empproc}{{\mathbb P}} 
\newcommand{\rproc}{{\mathbb R}} 
\DeclareMathOperator{\conv}{co}
\DeclareMathOperator{\lin}{lin}
\DeclareMathOperator{\bin}{bin}
\DeclareMathOperator{\hardtanh}{hardtanh}
\usepackage{color}

\newcommand{\funcgrad}{D}

\newcommand{\citet}[1]{\citeauthor{#1} \shortcite{#1}} 
\newcommand{\citep}{\cite}

\begin{document}
\maketitle

\begin{abstract}
	We consider learning a convex combination of basis models, and present some
	new theoretical and empirical results that demonstrate the
	effectiveness of a greedy approach.
	Theoretically, we first consider whether we can use linear, instead of
	convex, combinations, and obtain generalization results similar to existing
	ones for learning from a convex hull.
	We obtain a negative result that even the linear hull of very simple
	basis functions can have unbounded capacity, and is thus prone to overfitting;
	on the other hand, convex hulls are still rich but have bounded capacities.
	Secondly, we obtain a generalization bound for a general class of Lipschitz
	loss functions.
	Empirically, we first discuss how a convex combination can be greedily learned with
	early stopping, and how a convex combination can be non-greedily learned when
	the number of basis models is known a priori.
	Our experiments suggest that the greedy scheme is competitive with or better than
	several baselines, including boosting and random forests.
	The greedy algorithm requires little effort in hyper-parameter tuning, and
	also seems able to adapt to the underlying complexity of the problem.
	Our code is available at \url{https://github.com/tan1889/gce}.
\end{abstract}

\section{Introduction} \label{sec:intro}

Various machine learning methods combine given basis models to form richer models
that can represent more complex input-output relationships.
These include random forests \citep{breiman2001random} and boosting
\citep{freund1995desicion,mason2000functional,chen2016xgboost}, which have
often been found to work well in domains with good features.
Interestingly, even combining simple basis models like decision stumps can work
very well on hard problems \citep{viola2004robust}. 

In this paper, we consider learning an optimal convex combination
of basis models
\citep{lee1996efficient,mannor2003greedy,oglic2016greedy,wyner2017explaining},
and present new theoretical and empirical insights. 

We first compare learning from convex hulls with learning from the closely related
linear hulls.
Learning from a convex hull can be seen as a regularized version of learning
from the corresponding linear hull, where we enforce constraints on the weights
of the basis functions.
While linear hulls are known to provide universal approximations 
\citep{barron1993universal,makovoz1996random}, our analysis shows that they
can be prone to overfitting.
Specifically, we show that the capacity of the linear hull of very simple
functions can be unbounded, while the convex hull is still rich but has bounded
capacity.

Our second contribution is a generalization result for a general class of
Lipschitz loss functions.
A number of works studied algorithms for learning a convex combination and
analyzed their generalization performance.
However, previous works mostly focused on generalization performance with
quadratic loss \citep{lee1996efficient,mannor2003greedy} or large margin type
analysis \citep{koltchinskii2005complexities} for classification problems.
The quadratic loss is a special case of the class of Lipschitz loss functions
considered in this paper.
In addition, our result shows that we can obtain an $O(1/\sqrt{n})$ convergence
rate for log-loss in the classification setting.

Empirically, we present an extensive experimental evaluation of algorithms for
learning from convex hulls.
While previous works mainly focused on simple greedy algorithms to learn a convex
combination, we leverage on the functional optimization versions of some 
sophisticated algorithms to develop algorithms for learning a convex
combination.
In particular, we consider the Frank-Wolfe (FW) algorithm and its variants
\citep{jaggi2013revisiting}, which provide natural ways to build convex
combinations.
We also show how a convex combination can be non-greedily learned when the
number of basis functions is known a priori.
Our experiments suggest that the greedy scheme is competitive with or better than
several baselines, including boosting and random forests.
The greedy algorithm requires little hyper-parameter tuning, and
also seems to adapt to the underlying complexity of the problem.

\Cref{sec:related} further discusses related works.
\Cref{sec:theory} presents our theoretical analysis for learning from a convex
hull. 
\Cref{sec:algo} discusses some greedy learning algorithms, and a non-greedy
version. 
\Cref{sec:expt} empirically compares the algorithms for learning from convex hulls, and a
few baselines.
\Cref{sec:conclude} concludes the paper.

\section{Related Work} \label{sec:related}

A number of works have studied the generalization performance of algorithms for
learning convex combinations.
\citet{lee1996efficient} considered learning a convex combination of linear
threshold units with bounded fan-in (\#inputs) for binary classification using quadratic
loss, and they showed that an optimal convex combination is PAC-learnable.
\citet{mannor2003greedy} also considered binary classification, and obtained a 
generalization result for general basis functions and quadratic loss.
They also obtained a consistency result for more general loss functions.
\citet{koltchinskii2005complexities} provided some generalization results for
learning a convex combination by maximizing margin.
\citet{oglic2016greedy} considered regression with quadratic loss and presented
a generalization analysis for learning a convex combination of cosine ridge
functions.
We obtained generalization bounds for a class of Lipschitz loss functions and
general basis functions.

Various authors considered greedy approaches for learning a convex combination,
which iteratively constructs a convex combination by choosing a good convex
combination of the previous convex combination and a new basis function.
\citet{jones1992simple} presented a greedy algorithm and showed that it
converges at $O(1/k)$ rate for quadratic loss.
This is further developed by \citet{lee1996efficient} and
\citet{mannor2003greedy}. 
\citet{zhang2003sequential} generalized these works to convex functionals.
We leverage on the functional versions of more sophisticated greedy
optimization algorithms; in particular, the FW
algorithm and its variants, 
which have recently attracted significant attention in the numerical optimization
literature \citep{jaggi2013revisiting}.
Recently, \citet{bach2017breaking} considered using the FW algorithm to learn
neural networks with non-Euclidean regularizations, and showed that the
sub-problems can be NP-hard.
Besides greedy approaches, we show how to non-greedily learn a convex combination given
the number of basis functions.
We empirically compared the effectiveness of these algorithms.

Works on random forests \citep{breiman2001random} and boosting
\citep{mason2000boosting} are also closely related.
A random forest can be viewed as a convex combination of trees independently
trained on bootstrap samples, where the trees have equal weights.
Boosting algorithms greedily construct a conic, instead of convex, combination
of basis functions, but for binary classification, a conic combination can be
converted to a convex combination without changing the predictions.
There are numerous related works on the generalization performance of
boosting
(e.g. see \citep{bartlett2007adaboost,schapire2013explaining,gao2013doubt}).
Random forests are still less well understood theoretically yet
\citep{wyner2017explaining}, and analysis can require unnatural 
assumptions \citep{wager2015adaptive}.
We empirically compared algorithms for learning a convex combination with random
forests and boosting.

There have been also several recent applications of boosting for
generative models. Specifically, \citet{locatello2018boosting} show that boosting
variational inference satisfies a relaxed smoothness assumption which is
sufficient for the convergence of the functional Frank-Wolfe algorithm;
\citet{grover2018boosted} consider Bayes optimal classification; and \citet{tolstikhin2017adagan}
propose AdaGAN, which is adapted from AdaBoost for Generative Adversarial
Networks. Our work is orthogonal to these works in the sense that we study
discriminative models and learning from a convex hull, instead of a linear or a
conic hull. Moreover, our bound might be interesting in comparison to vacuous bounds that
grow rapidly in the number of parameters: while the number of
parameters for a convex combination can be unbounded, our error bound depends
only on the pseudodimension of the basis models.

\section{Theoretical Analysis} \label{sec:theory}

Given an i.i.d. sample 
$\bz = ((x_{1}, y_{1}), \ldots, (x_{n}, y_{n}))$ drawn from a distribution
$P(X, Y)$ defined on $\setX \times \setY \subseteq \setX \times \realnum$,
we want to learn a function $f$ to minimize the risk 
$R(f) = \E L(Y, f(X))$, where $L(y, \hat{y})$ is the loss that $f$ incurs when
predicting $y$ as $\hat{y}$, and the expectation is taken wrt $P$.
The empirical risk of $f$ is
$R_{n}(f) 
= \E_{n} L(Y, f(X))
= \frac{1}{n} \sum_{i=1}^{n} L(y_{i}, f(x_{i}))$.
Without loss of generality, we assume $\setX \subseteq \realnum^{d}$.

Given a class of basis functions $\setG \subseteq \setY^{\setX}$, we use
$\conv_{k}(\setG)$ to denote the set of convex combinations of $k$ functions in
$\setG$,
that is,
$\conv_{k}(\setG) 
	= \{\sum_{i=1}^{k} \alpha_{i} g_{i}:
	\sum_{i} \alpha_{i} = 1, 
	\text{each } \alpha_{i} \ge 0,
	\text{each } g_{i} \in \setG\}$.
The convex hull of $\setG$ is 
$\conv(\setG) = \cup_{k \ge 1} \conv_{k}(\setG)$.
We will also use $\lin_{k}(\setG)$ to denote the set of linear combinations of $k$
functions in $\setG$, that is,
$\lin_{k}(\setG) = \{\sum_{i=1}^{k} \alpha_{i} g_{i}: 
\alpha_{1}, \ldots, \alpha_{k} \in \realnum,
g_{1}, \ldots, g_{k} \in \setG\}$.
The linear hull of $\setG$ is $\lin(\setG) = \cup_{k \ge 1} \lin_{k}(\setG)$.
The basis functions are assumed to be bounded, with 
$\setG \subseteq \setY^{\setX} \subseteq [-B, B]^{\setX}$ 
for some constant $B > 0$. 

\medskip\noindent{\bf Capacity measures}.
A function class needs to be rich to be able to fit observed data, but cannot be
too rich so as to make generalization possible, that is, it needs to have the
right \emph{capacity}.
Commonly used capacity measures include VC-dimension, pseudodimension, and
Rademacher complexity.

VC-dimension is defined for binary valued functions.
Specifically, for a class $\setF$ of binary valued functions, its VC-dimension
$d_{VC}(\setF)$ is the largest $m$ such that there exists $m$ examples $x_{1}, \ldots, x_{m}$
such that the restriction of $\setF$ to these examples contains $2^{m}$
functions.
Equivalently, for any $y_{1}, \ldots, y_{m} \in \{0, 1\}$, there is a function
$f \in \setF$ such that $f(x_{i}) = y_{i}$ for all $i$.
$x_{1}, \ldots, x_{m}$ is said to be shattered by $\setF$.

Pseudodimension \citep{pollard1984convergence} is a generalization of
VC-dimension to real-valued functions.
The pseudodimension $d_{P}(\setF)$ of a class of real-valued functions $\setF$
is defined as the maximum number $m$ such that there exists $m$ inputs 
$x_{1}, \ldots, x_{m} \in \setX$, and thresholds 
$t_{1}, \ldots, t_{m} \in \realnum$
satisfying 
$\{(\I(f(x_{1}) \ge t_{1}), \ldots, \I(f(x_{m}) \ge t_{m})): f \in \setF\}
= \{0, 1\}^{m}$.
If each $f \in \setF$ is binary-valued, then $d_{P}(\setF) = d_{VC}(\setF)$.

Rademacher complexity is
defined as $\E \sup_{f \in \setF} \rproc_{n} f$, where 
$\rproc_{n}$ is the Rademacher process defined by 
$\rproc_{n} f = \frac{1}{n} \sum_{i} \epsilon_{i} f(x_{i}, y_{i})$, with
$(x_{i}, y_{i})$'s being an i.i.d. sample, and $\epsilon_{i}$'s being independent
Rademacher random variables (i.e., they have probability 0.5 to be -1 and 1).
Expectation is taken wrt both the random sample and the Rademacher
variables.

We refer the readers to the book of \citet{anthony2009neural} and the article of
\citet{mendelson2003few} for excellent discussions on these capacity measures
and their applications in generalization results. 

\subsection{A Regularization Perspective}
Several authors showed that linear hulls of various basis functions are 
universal approximators \citep{barron1993universal,makovoz1996random}.
Naturally, one would like to learn using linear hulls if possible.
On the other hand, the richness of the linear hulls also imply that they may be
prone to overfitting, and it may be beneficial to consider regularization.

Learning from the convex hull can be seen as a regularized version of learning
from the linear hull, where the regularizer is 
$\I_{\infty}(f) = \begin{cases}
	0, & f \in \conv(\setG), \\
	\infty, & \text{otherwise}.
\end{cases}$.
This is similar to $\ell_{2}$ regularization in the sense that $\ell_{2}$
regularization constrained the weights to be inside an $\ell_{2}$ ball, while
here we constrain the weights of the basis model to be inside a simplex.
A key difference is that standard $\ell_{2}$ regularization is often applied to
a parametric model with fixed number of parameters, but here the number of
parameters can be infinite.

We compare the capacities of the linear hull and the convex hull of a class of
basis functions $\setG$ with finite pseudodimension, and demonstrate the
effect of the regularizer $\I_{\infty}$ in controlling the capacity:
while the convex hull can still be rich, it has a more adequate capacity for
generalization.

For a class of functions $\setF$, we shall use 
$\bin(\setF) = \{x \mapsto \I(f(x) \ge t): f \in \setF, t \in \realnum\}$ to
denote the thresholded binary version of $\setF$.
Consider the set of linear threshold functions
$\setT = \{\I(\theta^{\top} x \ge t): \theta \in \realnum^{d}, t \in \realnum\}$.
It is well-known that the VC-dimension of the thresholded versions of the linear
combination of $k$ linear threshold functions can grow quickly.
\vspace{-1.2em}
\begin{prop} (\citep{anthony2009neural}, Theorem 6.4)
	The VC-dimension of $\bin(\lin_{k}(\setT))$ is at least 
	$\frac{dk}{8} \log_{2}\left(\frac{k}{4}\right)$ 
	for $d>3$ and $k \le 2^{d/2 - 2}$.
\end{prop}
The above result implies that $d_{P}(\lin_{k}(\setT))$ is at least 
$\frac{dk}{8} \log_{2}\left(\frac{k}{4}\right)$.
A natural question is whether the VC-dimension still grows linearly when 
$k > 2^{d/2-2}$.
We give an affirmative answer via a constructive proof, and provide counterpart
results for the convex hull.
\begin{restatable}{prop}{prophull} \label{prop:hull}
	(a) Assume that $d \ge 2$.
	Then $d_{VC}(\bin(\lin_{k}(\setT))) \ge k$, thus 
	$d_{P}(\lin_{k}(\setT)) \ge k$, 
	and $d_{P}(\lin(\setT)) = \infty$.
	In addition, the Rademacher complexity of $\lin(\setT)$ is infinite. \\
	(b) Assume that $d \ge 2$.
	Then $d_{VC}(\bin(\conv_{k+1}(\setT))) \ge k$, thus
	$d_{P}(\conv_{k+1}(\setT)) \ge k$, and 
	$d_{P}(\conv(\setT)) = \infty$, but the Rademacher complexity of
	$\conv(\setT)$ is finite.
\end{restatable}
\vspace{-0.9em}
\begin{proof}
	(a) Consider an arbitrary unit circle centered at the origin, and any $k$ points
	$x_{1}, \ldots, x_{k}$ which are equally spaced on the circle.
	Let $\theta_{i} = x_{i}$ and $b_{i} = 1$ for $i = 1, \ldots, k$.
	For any $y_{1}, \ldots, y_{k} \in \{0, 1\}$, consider the linear
	combination 
		$f(x) = \sum_{i=1}^{k} w_{i} \I(\theta_{i}^{\top} x \ge b_{i})$, with $w_{i} = y_{i}$.
	We have $f(x_{i}) = y_{i}$.
	The classifier $t(x) = \I(f(x_{i}) \ge 1)$ is a thresholded classifier
	obtained from $f$, and thus $t \in \bin(\lin_{k}(\setT))$.
	In addition, $t(x) = \I(y_{i} \ge 1) = y_{i}$.
	In short, for any $y_{1}, \ldots, y_{k} \in \{0, 1\}$, there is a classifier
	$t \in \bin(\lin_{k}(\setT))$ such that $t(x_{i}) = y_{i}$.
	Thus $d_{VC}(\bin(\lin_{k}(\setT))) \ge k$.
	It follows that $d_{P}(\lin_{k}(\setT)) \ge k$, and thus 
	$d_{P}(\lin(\setT)) = \infty$.
	
	The Rademacher complexity of $\lin(\setT)$ is infinity, because for any $c >
	0$, the Rademacher complexity of $c \lin(\setT)$ is $c$ times that of 
	$\lin(\setT)$.
	On the other hand, $c \lin(\setT) = \lin(\setT)$.
	Hence the Rademacher complexity of $\lin(\setT)$ can be arbitrarily large, and
	is thus infinity.
	
	\medskip\noindent(b) Consider the function
	$h(x) 
	= w_{0} \I({\bf 0}^{\top} x \ge 0.5) 
	+ \sum_{i=1}^{k} w_{i} \I(\theta_{i}^{\top} x \ge b_{i})$, where
	$w_{0} = 1/(1 + \sum_{i} y_{i})$ and $w_{i} = y_{i}/(1 + \sum_{i} y_{i})$ for
	$i \ge 1$.
	The function $h(x)$ is a convex combination of 
	$\I({\bf 0}^{\top} x \ge 0.5),
	\I(\theta_{1}^{\top} x \ge b_{1}), \ldots,
	\I(\theta_{k}^{\top} x \ge b_{k})$, where the first one is always 0.
	For any $x_{j}$, we have $h(x_{j}) = y_{j}/(1 +\sum_{i} y_{i}) \ge y_{j} /
	(k+1)$, because each $y_{i}$ is either 0 or 1.
	Hence we have $\I(h(x_{j}) \ge  1/(k+1)) = y_{j}$.
	It follows that $x_{1}, \ldots, x_{k}$ can be shattered by the thresholded version of
	$\conv_{k+1}(\setT)$.

	The Rademacher complexity of the convex hull is equal to that of $\setT$
	according to Theorem 2.25 in \citep{mendelson2003few}, which is finite as
	$d_{VC}(\setT) = d+1$ is finite.
\end{proof}

\Cref{prop:hull} shows that the linear hull has infinite capacity, both in terms
of pseudodimension and in terms of Rademacher complexity. 
Thus, it may easily overfit a training dataset. 
On the other hand, the convex hull is more restricted with a finite Rademacher complexity, but still rich because it has infinite pseudodimension. 
This can be attributed to regularization effect imposed by the convex coefficients constraints.

\subsection{Generalization Error Bounds}
Let $f^{*}(x) = \min_{y \in \setY} \E [L(y, Y) | X=x]$
be the Bayes optimal function .
For binary classification problems (that is, $\setY = \{-1, 1\}$) using the
quadratic loss $L(y, f(x)) = (f(x) - y)^{2}$, \citet{mannor2003greedy} obtained
the following uniform convergence rate with an assumption on the uniform entropy
$H(\epsilon, \conv(\setG), n)$ of $\conv(\setG)$, where 
$\setG \subseteq [-B, B]^{\setX}$.
\begin{thm}  \label{thm:convbayes0} (Theorem 9 in \citep{mannor2003greedy})
	Assume that for all positive $\epsilon$, $H(\epsilon, \conv(\setG), n) \le
	K(2B/\epsilon)^{2\xi}$ for $\xi \in (0, 1)$ and $K > 0$.
	Then there exist constants $c_{0}, c_{1} > 0$ that depend on $\xi$ and
	$K$ only, such that $\forall \delta \ge c_{0}$, with probability at least 
	$1 - e^{-\delta}$, for all $f \in \conv(\setG)$,
	\resizebox{.93\linewidth}{!}{\parbox{\linewidth}{%
	\begin{align*}
		R(f) - R(f^{*}) \le 4 \left(R_{n}(f) - R_{n}(f^{*})\right)  
			+ \frac{c_{1} 4 B^{2}}{\xi}
			\left(\frac{\delta}{n}\right)^{1/(1 + \xi)}.
	\end{align*}
	}}
\end{thm}
When $d_{P}(\setG) = p$, then the assumption on the metric entropy
$H(\epsilon, \conv(\setG), n)$ is satisfied with 
$\xi = \frac{p}{p + 2}$ \citep{wellner2002upper}. 
In \Cref{thm:convbayes}, we prove a more general bound for a class of Lipschitz losses that includes the quadratic loss considered in \Cref{thm:convbayes0} as a special case.
Omitted proofs are available 
\fullversion{in the appendix}
{at \url{https://github.com/tan1889/gce}}.
\begin{restatable}{thm}{thmconvbayes} \label{thm:convbayes}
	Assume that $d_{P}(\setG) = p < \infty$, and 
	$L(y, f(x)) = \phi(f(x) - y)$ for a $c_{\phi}$-Lipschitz nonnegative function
	$\phi$ satisfying $\phi(0) = 0$.
	With probability at least $1 - \delta$, for all $f \in \conv(\setG)$,
		$R(f) - R(f^{*})
		\le R_{n} (f) - R_{n}(f^{*}) 
			+ \frac{c}{\sqrt{n}}$,
	where 
	$c = {2 c_{\phi} B \left(\sqrt{2 \ln(1/\delta)} + D\sqrt{p} + 2\right)}$, 
	and $D$ is an absolute constant.
\end{restatable}
The Bayes optimal function $f^{*}$ is generally not in $\conv(\setG)$, thus the chosen
convex combination $f$ may not reach the level of performance of
$f^{*}$.
Thus we are often interested in the convergence of the empirical minimizer to the
optimal model in $\conv(\setG)$.
We can obtain an $O(1/\sqrt{n})$ convergence rate by closely
following the proof of \Cref{thm:convbayes}.
\begin{restatable}{thm}{thmconvclass} \label{thm:convclass}
	Assume that $d_{P}(\setG) = p < \infty$, 
	$L(y, f(x)) = \phi(f(x) - y)$ for a $c_{\phi}$-Lipschitz nonnegative function
	$\phi$ satisfying $\phi(0) = 0$.
	Let $\hat{f} = \argmin_{f \in \conv(\setG)} R_{n}(f)$, and 
	$h^{*} = \argmin_{f \in \conv(\setG)} R(f)$, then with probability at
	least $1 - \delta$,
		$R(\hat{f}) \le R(h^{*}) + \frac{c}{\sqrt{n}}$,
	where $c = {2 c_{\phi} B \left(\sqrt{2 \ln(1/\delta)} + D\sqrt{p} + 2\right)}$, 
\end{restatable}

As a special case, we have the result below for $\ell_{q}$ regression.
\vspace{-1.2em}
\begin{restatable}{cor}{corlp}
	For $L(y, f(x)) = |f(x) - y|^{q}$, $q \ge 1$, the bounds in
	\Cref{thm:convbayes} and \ref{thm:convclass} hold with 
	$c_{\phi} = q (2B)^{q-1}$.
\end{restatable}
\citet{donahue1997rates} showed that tighter bounds can be obtained for
$\ell_{p}$ regression by exploiting the specific form of $\ell_{p}$ loss.
Our analysis provides a looser bound, but is simpler and can be applied to the
classification setting below.

Specifically, for binary classification with $\setY = \{-1, 1\}$, we can also obtain an
$O(1/\sqrt{n})$ generalization bound for a class of Lipschitz loss as a
corollary of the proof of \Cref{thm:convbayes}.
The loss in this case is Lipschitz in $y f(x)$ (not $f(x) - y$ as in the
regression case), with a positive value indicating that $f(x)$ is better aligned
with $y$.
\begin{restatable}{cor}{corcls} \label{cor:cls}
	Assume that $d_{P}(\setG) = p < \infty$, $\setY = \{-1, 1\}$,
	$L(y, f(x)) = \phi(y f(x))$ for a $c_{\phi}$-Lipschitz nonnegative function
	$\phi$ satisfying $\phi(0) = 0$.
	Let $\hat{f} = \argmin_{f \in \conv(\setG)} R_{n}(f)$, and 
	$h^{*} = \argmin_{f \in \conv(\setG)} R(f)$, then with probability at
	least $1 - \delta$,
		$R(\hat{f}) \le R(h^{*}) + \frac{c}{\sqrt{n}}$,
	where $c = {2 c_{\phi} B \left(\sqrt{2 \ln(1/\delta)} + D\sqrt{p} + 2\right)}$, 
\end{restatable}

As a special case, the above rate holds for the log-loss.
\begin{restatable}{cor}{corlog}
	When $y \in \{-1, 1\}$, 
	$L(y, f(x)) = -\ln\frac{1}{1+e^{-y f(x)}}$, 
	the bound in \Cref{cor:cls} holds with $c_{\phi} = 1$.
\end{restatable}

\section{Algorithms} \label{sec:algo}
We consider algorithms for finding the empirical risk minimizer 
$\hat{f} = \arg\min_{f \in \conv(\setG)} R_n(f)$, when
$\setG$ consists of parametric basis models, i.e.,
$\setG = \{g_{\theta}: \theta \in \realnum^{p}\}$, where $g_{\theta}$
denotes a model with parameters $\theta$. 

\subsection{Greedy Algorithms}
The convexity of $\conv(\setG)$ allows the following greedy scheme.
Start with some $f_{0} \in \setH$.
At iteration $k$, we choose appropriate $\alpha_{t} \in [0,1]$ and 
$g_{t} \in \setG$, for the new convex combination
\begin{align}
	f_{t+1} = (1- \alpha_{t}) f_{t} + \alpha_{t} g_{t}.
\end{align}
We run the algorithm up to a maximum number of iterations $T$, or do
\emph{early stopping} if the improvements in the last few iterations is
negligible (less than a small threshold).

Such scheme generates sparse solutions in the sense that at iteration $t$, the
convex combination consists of at most $t$ basis functions, even though the
optimal combination can include arbitrarily large number of basis functions.

We present several instantiations of this scheme based on results from
functional optimization.
The derived sub-problems, while being functional optimization problems, are
equivalent to finite-dimensional numerical optimizations which can be solved
using stochastic gradient descent.
Some of them have interesting forms that differ from standard risk minimization
problems.

\SetKwFunction{RiskMinimizerOracle}{RiskMinimizerOracle}
\SetKwFunction{LinearMinimizerOracle}{LinearMinimizerOracle}

\medskip\noindent{\bf A nonlinear greedy algorithm}.
One natural way to choose $g_{t}$ and $\alpha_{t}$ is to choose them jointly so as
to maximize the decrease in the empirical risk
\citep{jones1992simple,lee1996efficient,mannor2003greedy,zhang2003sequential}.
Specifically,
\begin{align}
	\hspace{-1em}
	\theta_{t}, \alpha_{t} 
	\gets 
	\argmin_{\theta \in \realnum^{p}, \alpha \in [0, 1]}
		R_{n}((1 - \alpha) f_{t-1} + \alpha g_{\theta})
\end{align}
For common loss functions, the RHS is usually a differentiable function of
$\theta$ and $\alpha$, and thus the problem can be solved using first-order
methods.
We used Adam \citep{kingma2014adam} in our experiments.

When $L(y, f(x))$ is convex and smooth functional of $f$, it is known, e.g. from
\citep{zhang2003sequential}, that 
	$R_{n}(f_{t}) - R_{n}(\hat{f}) \le O(1/t)$.
In fact, we can still achieve a convergence rate of $O(1/t)$, as long as we can
solve the greedy step with an error of $O(1/t^{2})$, that is, if we can choose $g_{t}$
and $\alpha_{t}$ such that 
	$R_{n}((1 - \alpha_{t}) f_{t-1} + \alpha_{t} g_t)
	\le \min_{g \in \setG, \alpha \in [0, 1]} 
			R_{n}((1 - \alpha) f_{t-1} + \alpha g)
			+ \frac{c}{t^{2}}$,
for some constant $c > 0$ \citep{zhang2003sequential}.
In particular, this result applies to the quadratic loss. 

\medskip\noindent{\bf The FW algorithm}.
The FW algorithm \citep{frank1956algorithm} does not choose $g_{t}$ to directly
minimize the risk functional at each iteration, but chooses it by solving a
linear functional minimization problem
\begin{align}
	g_{t} &= \arg\min_{g \in \setG} \langle \funcgrad R_{n}(f_{t-1}), g \rangle.
	\label{eq:lmo}
\end{align}
where the step size $\alpha_{t}$ can be taken as 
$\alpha_{t} = \frac{1}{t+1}$ or chosen using line search, and 
$\funcgrad R_{n}(f)$ denotes the functional gradient of $R_{n}$ with
respect to $f$, which is only non-zero at the points in the sample and is thus
finite.
For the quadratic loss $L(y, f(x)) = (f(x) - y)^{2}$,
$g_{t}$ is $g_{\theta_{t}}$ with $\theta_{t}$ chosen by
\begin{align}
	\theta_{t} = \arg\min_{\theta \in \realnum^{p}} 
		\sum_{i} 2 (f_{t-1}(x_{i}) - y_{i}) g_{\theta}(x_{i}).
\end{align}
This optimization problem has an interesting form different from standard risk
minimization problems.
For quadratic loss, we can also derive the closed form solution for the
line-search for $\alpha_t$.
For other risk criteria, there is no closed form solution, and we treat that as
a parameter in the numerical optimization problem in each iteration.
FW also converges at an $O(1/t)$ rate as for smooth and convex functionals (e.g.,
see \citet{jaggi2013revisiting}), with the constant in the rate being of the
same order as that for the nonlinear greedy algorithm.

\medskip\noindent{\bf Away-step and Pairwise FW}.
The away-step Frank-Wolfe (AFW) \citep{guelat1986some}, and the pairwise
Frank-Wolfe (PFW) \citep{lacoste2015global} are faster variants which can
converge at a linear rate when the solution is not at the boundary.

AFW either takes a standard FW step or an \emph{away} step which removes a basis
network from current convex combination and redistributes the weight to
remaining basis networks. 
Specifically, at each iteration, it finds $g_{t} \in \setG$ that is most aligned
with the negative gradient $\funcgrad R_{n}(f_{t-1})$ as in the FW algorithm, 
and a basis function $a_{t}$ that is most misaligned with the negative gradient
$\funcgrad R_{n}(f_{t-1})$ from the set of basis functions in $f_{t-1}$.
Here, the inner product of two vectors measures the degree of
alignment between them.
It then constructs a FW direction $d_t^{\text{FW}}$ that moves towards $g_{t}$,
and an away-step direction $d_{t}^{\text{A}}$ that moves away from $a_{t}$.
The direction that is better aligned with the negative gradient is taken.
For the away-step, the step size is restricted to be in 
$[0, \frac{\alpha_{a_{t}}}{1 - \alpha_{a_{t}}}]$ so that the weight of $a_{t}$
remains non-negative in $f_{t}$.

PFW swaps the weight of $a_{t}$ and $g_{t}$
determined in AFW by moving along the direction $g_{t} - a_{t}$.
Line search is used to determine the optimal step size.

\subsection{A Non-greedy Algorithm}\label{sec:nongreedy}
If we know the number of basis models required a priori, we can
train the weights of the basis models and the convex coefficients
simultaneously.
Instead of using constrained optimization techniques, we can use a softmax
normalization.
However, we found this not working well in practice.

We propose a simple unconstrained parametrization of the convex coefficients
that have been observed to perform well in our experiments.
Specifically, if we know that the number of basis model is $k$, we reparameterize the convex coefficients 
$\alpha_{1}, \ldots, \alpha_{k}$ as a function of the unconstrained parameter vector 
$v \in \realnum^{k}$ with $\alpha_i = \frac{1/k + |v_i|}{1 + \sum_{i=1}^k | v_i |}$.
The model $\sum_{i} \alpha_{i} g_{\theta_{i}}(x)$ can be seen as a neural
network that can be trained conventionally.

\subsection{Implementation}
In practice, a lot of interesting functions are unbounded, and we may need
multiple output.

We can convert unbounded functions to bounded ones by using the scaled hard tanh
$\hardtanh(y) = B\max(-1, \min(y, 1))$
to clamp the output to $[-B, B]$.
Sometimes it is beneficial to choose the scaling factor $B$ to be larger than
the actual possible range: when $\setG$ contains the zero function,
$\setG$ is a larger set when $B$ is larger.

When multiple outputs are needed, we can train a model for each output separately.
If the convex hull contains the true input-output function for each output
channel, then the generalization theory for the case of
single output guarantees that we will learn all the input-output functions
eventually.

\section{Experiments} \label{sec:expt}

We choose $g_{\theta}$ as a small neural network in our experiments.
We compare the performance of the greedy algorithms (nicknamed as GCE, which
stands for greedy convex ensemble) in \Cref{sec:algo} to study whether it is
beneficial to use more sophisticated greedy algorithms.
We also compare the greedy algorithms with XGBoost (XGB) and Random Forest (RF)
to study how the convex combination constructed, which can be seen as a weighted
ensemble, fare.
Both XGB and RF provide strong baselines and are state-of-the-art
ensemble methods that won many Kaggle competitions for non-CV non-NLP tasks.
In addition, we also compare the greedy algorithms with the non-greedy method in
\Cref{sec:nongreedy} (NGCE), and a standard regularized neural network (NN).
Both NGCE and NN need to assume a given number of basis functions.
Comparison with NN sheds light on how the regularization effect of learning from
a convex hull compare with standard $\ell_{2}$ regularization.

We used 12 datasets of various sizes and tasks (\#instances and dimensions in brackets): 
\emph{diabetes (442/10), boston (506/13), ca\_housing (20,640/8), msd
(515,345/90)} for regression; 
\emph{iris (150/4), wine (178/13), breast\_cancer (569/30), digits (1,797/64), 
cifar10\_f (60,000/342), mnist (70,000/784), covertype (581,012/54), kddcup99
(4,898,431/41)} for classification. 
Most of the datasets are from UCI ML Repository \citep{uci}. 
All data attributes are normalized to have zero mean and unit variance.  
Further details on the datasets, their training/validation/test splits,
hyper-parameter tuning and experimental setup are available 
\fullversion{in the appendix}
{at \url{https://github.com/tan1889/gce}}.

\subsection{Comparison of GCE Algorithms} \label{subsec:compare_algs}
\Cref{fig:algs} shows the training and test error (MSE) on the \emph{msd} dataset
for the four greedy algorithms discussed in \Cref{sec:algo}.
For each variant, we train $100$ modules, with each being a single neuron of the
form $B \tanh(u^T x)$.
Interestingly, the non-linear greedy variant, which is most commonly studied, is
significantly slower than other variants.
The PFW variant has the best performance. 
We observed similar behavior of the algorithms on other datasets and settings,
thus we only report the results for the PFW variant in subsequent experiments.

\begin{figure}
\begin{center}
    \includegraphics[width=\columnwidth]{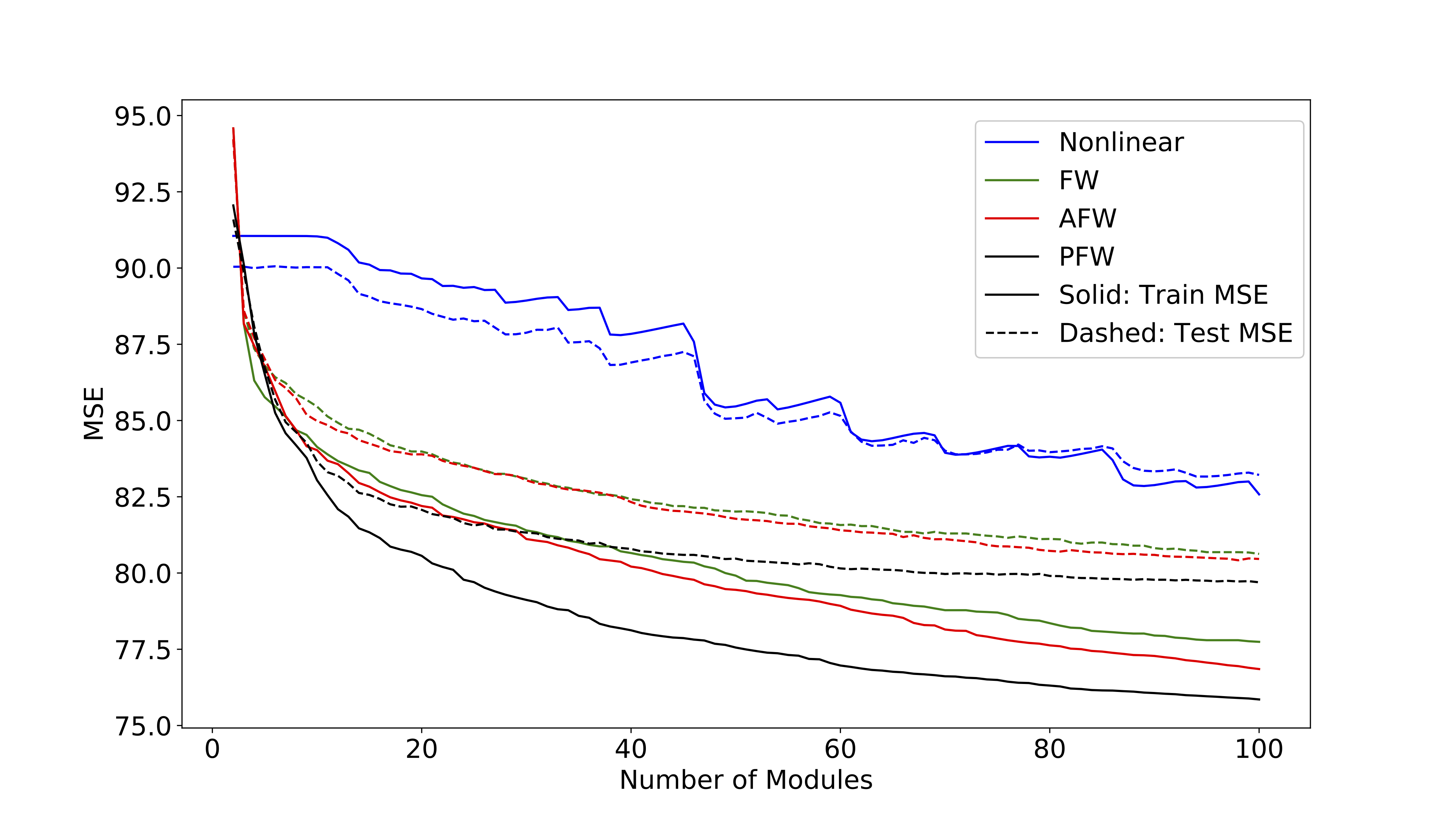}
    \caption{\label{fig:algs} \emph{Performance of different variants of the greedy algorithm on msd dataset.} 
		Mean squared error against number of modules added to the convex combination. 
		Solid and dashed curves indicate training and test errors respectively.}
\end{center}
\end{figure}

\subsection{Comparison of GCE with Other Algorithms}

We compare GCE (using PFW to greedily choose new basis functions), NGCE,
XGB, RF, and NN below.

\medskip\noindent{\bf Experimental setup}.
To ensure a fair comparison between algorithms, we spent a significant effort to tune hyper-parameters of competing algorithms. Particularly, XGB and RF are tuned over 2000 hyper-parameters combinations for small datasets (has less than 10000 training samples), and over 200 combinations for large datasets. 

The basis module for GCE is a two-layer network with 1 or 10 hidden neurons for small datasets  and 100 hidden neurons for other datasets. GCE grows the network by adding one basis module at a time until no improvement on validation set is detected, or until reaching the maximum limit of 100 modules. NGCE and NN are given the maximum capacity achievable by GCE. For NGCE, it is 100 modules, each of 10/100 hidden neurons for small/large datasets. NN is a two layers neural net of 1000/10000 hidden neurons for small/large datasets, respectively. NGCE and NN are tuned using grid search over \textit{learning\_rate} $\in \{0.01, 0.001\}$, \textit{regularization} $\in \{0, 10^{-6}, \dots, 10^{-1} \}$, totaling in 14 combinations. GCE uses a fixed set of hyper-parameters without tuning: \textit{learning\_rate}$=0.001$, \textit{regularization}$=0$. All these three algorithms use ReLU activation, MSE criterion for training regression problem, cross entropy loss for classification. The training uses Adam SGD with \textit{learning\_rate} reduced by 10 on plateau (training performance did not improve for 10 consecutive epochs) until reaching the minimum \textit{learning\_rate} of $10^{-5}$, at which point the optimizer is ran for another 10 epochs and then returns the  solution with the best performance across all training epochs.

\medskip\noindent{\bf Results and Discussion}.
For each algorithm, the model using the hyper-parameters with the best validation performance is selected.
Its test set performance is reported in \Cref{tbl:errors}. 
We ran the experiments on train/test splits fixed across all methods in
comparison, instead of using multiple train/test splits, which can be useful for
further comparison, but is beyond our computing resources.

\begin{table}[]
\resizebox{.87\linewidth}{!}{\parbox{\linewidth}{%
\centering
\begin{tabular}[t]{lrrrrrrr} 
\toprule
\textbf{Datasets} & \textbf{GCE} & \textbf{XGB} & \textbf{RF} & \textbf{NN} & \textbf{NGCE} \\
\midrule
diabetes & \textbf{42.706} & 46.569 & 49.519 & 43.283 & 44.703 \\
boston & \textbf{2.165} & 2.271 & 2.705 & 2.217 & 2.232 \\
ca\_housing & 0.435 & \textbf{0.393} & 0.416 & 0.440 & 0.437  \\
msd & \textbf{6.084} & 6.291 & 6.462 & 6.186 & 7.610 \\
\midrule
iris  & \textbf{0.00} & 6.67 & 6.67 & 3.33 & 10.00  \\
wine & \textbf{0.00} & 2.78 & 2.78 & \textbf{0.0} & \textbf{0.0}  \\
breast\_cancer & \textbf{3.51} & 4.39 & 8.77 & 3.51 & 4.39 \\
digits & 2.78 & 3.06 & \textbf{2.50} & 3.33 & 3.06 \\
cifar10\_f & \textbf{4.86} & 5.40 & 5.16 & 5.00 & 4.92 \\
mnist & 1.22 & 1.66 & 2.32 & 1.24 & \textbf{1.11} \\
covertype & 26.70 & 26.39 & 27.73 & 26.89 & \textbf{26.56} \\
kddcup99 & \textbf{0.01} & \textbf{0.01} & \textbf{0.01} & \textbf{0.01} & \textbf{0.01}\\
\bottomrule
\end{tabular}
}}
\caption{\label{tbl:errors} \emph{Summary of the empirical results.} 
	MAEs are reported for regression datasets (the first 4 lines), and
	misclassification rate (\%) are reported for classification datasets (last 8
	lines).
	XGBoost and RForest are tuned over more than 2000/200 hyper-parameters combinations for small/large datasets. NN and NGCE are tuned over 14 combinations. GCE grows the model by one basis module at a time and stops when no improvement is detected.}
\end{table}

From \Cref{tbl:errors}, GCE is competitive with the baselines on datasets with
different numbers of features and examples.
For regression problems, GCE uses up the maximum number of basis
functions, while for classification problems, GCE often terminates way earlier than
that, suggesting that it is capable to adapt to the complexity of the task.

We have chosen datasets from small to large in our experiments.
The small datasets can be easily overfitted, and are useful for comparing how
well the algorithms avoid overfitting.
Empirical results show that GCE builds up the convex ensemble to just a right
capacity, but not more. 
Thus, it has good generalization performance despite having almost little
parameter tuning, even for the smaller datasets, where overfitting could easily
occur.
On the other extreme, overfitting a very large datasets, like \textit{kddcup99},
is hard. So, empirical results for this dataset show that all models, including
ours, have adequate capacity, as they have similar generalization performance.
All algorithms were able to scale to the largest dataset (kddcup99), but GCE
usually requires little hyper-parameter tuning, while XGB and RF involve
significant tuning.

\noindent\emph{(a) Comparison of GCE against NN and NGCE}.
NN and NGCE have very similar performance as GCE on several datasets.
NN does not perform well on \emph{msd}, \emph{iris}, \emph{digits},
and \emph{cifar10\_f}, and NGCE does not perform well on
\emph{diabetes}, \emph{msd}, \emph{iris}, and \emph{breast\_cancer}.
In particular, both NN and NGCE perform poorly on \emph{iris}.
We suspect that NN and NGCE may be more susceptible to local minimum, leading
to an underfitted model, and we examined the training and test losses for GCE, NN
and NGCE. 
For several datasets, the differences between the losses of the three
algorithms are small, but large differences show up on some datasets.
On \emph{diabetes} and \emph{msd}, both NGCE and NN seem to underfit, with
much larger training and test losses than those of GCE.
NGCE and NN seem to overfit \emph{boston}, with similar test 
test losses as GCE, but much smaller training loss. 
NN seems to overfit on \emph{iris} as well.

NGCE is slightly poorer than NN overall.
An unregularized NN usually does not perform well.
Since NGCE is trained without any additional regularization (such as
$\ell_{2}$ regularization), this suggests that the convex hull constraint has
similar regularization effect as a standard regularization, and the improved
performance of GCE may be due to greedy training with early stopping.

GCE often learns a smaller model as compared to NGCE and NN on classification
problems and does not know the number of components to use a priori.
On the other hand, both NGCE and NN requires a priori knowledge of the number
of basis functions to use, which is set to be the maximum number of components
used for GCE.
Finding the best size for a given problem is often hard. 

\medskip\noindent\emph{(b) Comparison of GCE against XGBoost and RandomForest}.
GCE shows clear performance advantage over XGBoost and RandomForest on most
domains, except that 
XGBoost has a clear advantage on \emph{ca\_housing}, and RandomForest performing
slightly better on \emph{digits}.

While RandomForest and XGBoost have quite a few parameters to tune, and a proper
tuning often requires searching a large number of hyper-parameters, GCE works
well across datasets with a default setting for basis module optimization (no
tuning) and two options for module size.
Overall, GCE is often more efficient than RandomForest and XGBoost as there is
little tuning needed.
In addition, for large datasets, RandomForest and XGBoost are slow due to the
lack of mechanism for mini-batch training and no GPU speed-up. 

\section{Conclusion} \label{sec:conclude}

This paper presents some new results and empirical insights for learning an
optimal convex combination of basis models.
While we focused on the case with neural networks as the basis models, the
greedy algorithms in \Cref{sec:algo} can be applied with trees as basis models.
In the case of the nonlinear greedy algorithm and a quadratic loss, the greedy
step involves training a regression tree given a fixed step size.
For FW variants and other losses, the objective function at each greedy step
no longer corresponds to a standard loss function though.
Regarding generalization, since trees are nonparametric, our generalization
results do not hold, and we may face similar challenges as analyzing random
forests. 
There are interesting questions for further work.

\newpage
\bibliographystyle{named}
\bibliography{ref}

\begin{thebibliography}{}

\bibitem[\protect\citeauthoryear{Anthony and
  Bartlett}{2009}]{anthony2009neural}
Martin Anthony and Peter~L Bartlett.
\newblock {\em {Neural network learning: Theoretical foundations}}.
\newblock Cambridge University Press, 2009.

\bibitem[\protect\citeauthoryear{Bach}{2017}]{bach2017breaking}
Francis Bach.
\newblock Breaking the curse of dimensionality with convex neural networks.
\newblock {\em The Journal of Machine Learning Research}, 18(1):629--681, 2017.

\bibitem[\protect\citeauthoryear{Barron}{1993}]{barron1993universal}
Andrew~R Barron.
\newblock {Universal approximation bounds for superpositions of a sigmoidal
  function}.
\newblock {\em Information Theory, IEEE Transactions on}, 39(3):930--945, 1993.

\bibitem[\protect\citeauthoryear{Bartlett and
  Traskin}{2007}]{bartlett2007adaboost}
Peter~L Bartlett and Mikhail Traskin.
\newblock Adaboost is consistent.
\newblock {\em Journal of Machine Learning Research}, 8(Oct):2347--2368, 2007.

\bibitem[\protect\citeauthoryear{Breiman}{2001}]{breiman2001random}
Leo Breiman.
\newblock Random forests.
\newblock {\em Machine learning}, 45(1):5--32, 2001.

\bibitem[\protect\citeauthoryear{Chen and Guestrin}{2016}]{chen2016xgboost}
Tianqi Chen and Carlos Guestrin.
\newblock {XGBoost: A scalable tree boosting system}.
\newblock In {\em KDD}, pages 785--794. ACM, 2016.

\bibitem[\protect\citeauthoryear{Dheeru and Karra~Taniskidou}{2017}]{uci}
Dua Dheeru and Efi Karra~Taniskidou.
\newblock {UCI} machine learning repository, 2017.

\bibitem[\protect\citeauthoryear{Donahue \bgroup \em et al.\egroup
  }{1997}]{donahue1997rates}
Michael~J Donahue, C~Darken, Leonid Gurvits, and Eduardo Sontag.
\newblock {Rates of convex approximation in non-Hilbert spaces}.
\newblock {\em Constructive Approximation}, 13(2):187--220, 1997.

\bibitem[\protect\citeauthoryear{Frank and Wolfe}{1956}]{frank1956algorithm}
Marguerite Frank and Philip Wolfe.
\newblock {An algorithm for quadratic programming}.
\newblock {\em Naval research logistics quarterly}, 3(1-2):95--110, 1956.

\bibitem[\protect\citeauthoryear{Freund and
  Schapire}{1995}]{freund1995desicion}
Yoav Freund and Robert~E Schapire.
\newblock {A desicion-theoretic generalization of on-line learning and an
  application to boosting}.
\newblock In {\em {European conference on computational learning theory}},
  pages 23--37. Springer, 1995.

\bibitem[\protect\citeauthoryear{Gao and Zhou}{2013}]{gao2013doubt}
Wei Gao and Zhi-Hua Zhou.
\newblock On the doubt about margin explanation of boosting.
\newblock {\em Artificial Intelligence}, 203:1--18, 2013.

\bibitem[\protect\citeauthoryear{Grover and Ermon}{2018}]{grover2018boosted}
Aditya Grover and Stefano Ermon.
\newblock Boosted generative models.
\newblock In {\em Thirty-Second AAAI Conference on Artificial Intelligence},
  2018.

\bibitem[\protect\citeauthoryear{Gu{\'e}lat and
  Marcotte}{1986}]{guelat1986some}
Jacques Gu{\'e}lat and Patrice Marcotte.
\newblock {Some comments on Wolfe's `away step'}.
\newblock {\em Mathematical Programming}, 35(1):110--119, 1986.

\bibitem[\protect\citeauthoryear{Jaggi}{2013}]{jaggi2013revisiting}
Martin Jaggi.
\newblock {Revisiting Frank-Wolfe: Projection-free sparse convex optimization}.
\newblock In {\em {ICML}}, 2013.

\bibitem[\protect\citeauthoryear{Jones}{1992}]{jones1992simple}
Lee~K. Jones.
\newblock {A simple lemma on greedy approximation in Hilbert space and
  convergence rates for projection pursuit regression and neural network
  training}.
\newblock {\em The annals of Statistics}, pages 608--613, 1992.

\bibitem[\protect\citeauthoryear{Kingma and Ba}{2014}]{kingma2014adam}
Diederik~P Kingma and Jimmy Ba.
\newblock Adam: A method for stochastic optimization.
\newblock {\em arXiv preprint arXiv:1412.6980}, 2014.

\bibitem[\protect\citeauthoryear{Koltchinskii and
  Panchenko}{2005}]{koltchinskii2005complexities}
Vladimir Koltchinskii and Dmitry Panchenko.
\newblock Complexities of convex combinations and bounding the generalization
  error in classification.
\newblock {\em The Annals of Statistics}, 33(4):1455--1496, 2005.

\bibitem[\protect\citeauthoryear{Koltchinskii}{2001}]{koltchinskii2001rademacher}
Vladimir Koltchinskii.
\newblock Rademacher penalties and structural risk minimization.
\newblock {\em IEEE Transactions on Information Theory}, 47(5):1902--1914,
  2001.

\bibitem[\protect\citeauthoryear{Lacoste-Julien and
  Jaggi}{2015}]{lacoste2015global}
Simon Lacoste-Julien and Martin Jaggi.
\newblock {On the global linear convergence of Frank-Wolfe optimization
  variants}.
\newblock In {\em {NIPS}}, 2015.

\bibitem[\protect\citeauthoryear{Lee \bgroup \em et al.\egroup
  }{1996}]{lee1996efficient}
Wee~Sun Lee, Peter~L Bartlett, and Robert~C Williamson.
\newblock {Efficient agnostic learning of neural networks with bounded fan-in}.
\newblock {\em Information Theory, IEEE Transactions on}, 42(6):2118--2132,
  1996.

\bibitem[\protect\citeauthoryear{Locatello \bgroup \em et al.\egroup
  }{2018}]{locatello2018boosting}
Francesco Locatello, Gideon Dresdner, Rajiv Khanna, Isabel Valera, and Gunnar
  R{\"a}tsch.
\newblock Boosting black box variational inference.
\newblock In {\em Advances in Neural Information Processing Systems}, 2018.

\bibitem[\protect\citeauthoryear{Makovoz}{1996}]{makovoz1996random}
Yuly Makovoz.
\newblock Random approximants and neural networks.
\newblock {\em Journal of Approximation Theory}, 85(1):98--109, 1996.

\bibitem[\protect\citeauthoryear{Mannor \bgroup \em et al.\egroup
  }{2003}]{mannor2003greedy}
Shie Mannor, Ron Meir, and Tong Zhang.
\newblock Greedy algorithms for classification--consistency, convergence rates,
  and adaptivity.
\newblock {\em Journal of Machine Learning Research}, 4(Oct):713--742, 2003.

\bibitem[\protect\citeauthoryear{Mason \bgroup \em et al.\egroup
  }{2000a}]{mason2000functional}
Llew Mason, Jonathan Baxter, Peter~L Bartlett, Marcus Frean, et~al.
\newblock {Functional gradient techniques for combining hypotheses}.
\newblock In {\em {Advances in Large Margin Classifiers}}, pages 221--246. MIT,
  2000.

\bibitem[\protect\citeauthoryear{Mason \bgroup \em et al.\egroup
  }{2000b}]{mason2000boosting}
Llew Mason, Jonathan Baxter, Peter~L Bartlett, and Marcus~R Frean.
\newblock Boosting algorithms as gradient descent.
\newblock In {\em {NIPS}}, pages 512--518, 2000.

\bibitem[\protect\citeauthoryear{Mendelson}{2003}]{mendelson2003few}
Shahar Mendelson.
\newblock A few notes on statistical learning theory.
\newblock In {\em Advanced lectures on machine learning}, pages 1--40.
  Springer, 2003.

\bibitem[\protect\citeauthoryear{Oglic and G\"{a}rtner}{2016}]{oglic2016greedy}
Dino Oglic and Thomas G\"{a}rtner.
\newblock Greedy feature construction.
\newblock In {\em NeurIPS}, 2016.

\bibitem[\protect\citeauthoryear{Paszke \bgroup \em et al.\egroup
  }{2017}]{paszke2017automatic}
Adam Paszke, Sam Gross, Soumith Chintala, Gregory Chanan, Edward Yang, Zachary
  DeVito, Zeming Lin, Alban Desmaison, Luca Antiga, and Adam Lerer.
\newblock Automatic differentiation in pytorch.
\newblock In {\em NIPS-W}, 2017.

\bibitem[\protect\citeauthoryear{Pollard}{1984}]{pollard1984convergence}
David Pollard.
\newblock {\em Convergence of stochastic processes}.
\newblock Springer, 1984.

\bibitem[\protect\citeauthoryear{Schapire}{2013}]{schapire2013explaining}
Robert~E Schapire.
\newblock Explaining adaboost.
\newblock In {\em Empirical inference}, pages 37--52. Springer, 2013.

\bibitem[\protect\citeauthoryear{Tolstikhin \bgroup \em et al.\egroup
  }{2017}]{tolstikhin2017adagan}
Ilya~O Tolstikhin, Sylvain Gelly, Olivier Bousquet, Carl-Johann Simon-Gabriel,
  and Bernhard Sch{\"o}lkopf.
\newblock Adagan: Boosting generative models.
\newblock In {\em NeurIPS}, 2017.

\bibitem[\protect\citeauthoryear{Viola and Jones}{2004}]{viola2004robust}
Paul Viola and Michael~J Jones.
\newblock {Robust real-time face detection}.
\newblock {\em {IJCV}}, 57(2):137--154, 2004.

\bibitem[\protect\citeauthoryear{Wager and Walther}{2015}]{wager2015adaptive}
Stefan Wager and Guenther Walther.
\newblock Adaptive concentration of regression trees, with application to
  random forests.
\newblock {\em arXiv preprint arXiv:1503.06388}, 2015.

\bibitem[\protect\citeauthoryear{Wellner and Song}{2002}]{wellner2002upper}
Jon~A. Wellner and Shuguang Song.
\newblock {An upper bound for uniform entropy numbers}, 2002.

\bibitem[\protect\citeauthoryear{Wyner \bgroup \em et al.\egroup
  }{2017}]{wyner2017explaining}
Abraham~J Wyner, Matthew Olson, Justin Bleich, and David Mease.
\newblock Explaining the success of adaboost and random forests as
  interpolating classifiers.
\newblock {\em The Journal of Machine Learning Research}, 18(1):1558--1590,
  2017.

\bibitem[\protect\citeauthoryear{Zhang}{2003}]{zhang2003sequential}
Tong Zhang.
\newblock Sequential greedy approximation for certain convex optimization
  problems.
\newblock {\em IEEE Transactions on Information Theory}, 49(3):682--691, 2003.

\end{thebibliography}

\fullversion{
\clearpage
\appendix
\section{Proofs}
For a real-valued function $r$ defined on a set $\setZ$, and a random variable
$Z$ taking values from $\setZ$, we use 
$\empproc r$ to denote $\E r(Z)$, and 
$\empproc_{n} r$ to denote $\E_{n} r(Z)$.

\begin{lem} \label{lem1} (\citet{koltchinskii2001rademacher}, Lemma 2.5)
	Let $\setR$ be a class of real-valued functions defined on $\setZ$, then
	\begin{align*}
		\E \sup_{r \in \setR} |\empproc r - \empproc_{n} r| \le 2\E \sup_{r \in \setR} \rproc_{n} r.
	\end{align*}
\end{lem}

\begin{lem} \label{lem2} (\citet{mendelson2003few}, Theorem 2.25)
	\begin{itemize}
		\item[(a)]
			If $\phi: \realnum \to \realnum$ is $c_{\phi}$-Lipschitz and $\phi(0)
			= 0$, then 
			\begin{align*}
				\rproc_{n} \phi \circ r \le 2 c_{\phi} \rproc_{n} r.
			\end{align*}
		\item[(b)]
			If $g: \setZ \to [-M, M]$, then 
			\begin{align*}
				\E \sup_{r \in \setR} \rproc_{n} (r + g) 
				\le 
				\E \sup_{r \in \setR} \rproc_{n} r 
				+
				\frac{M}{\sqrt{n}}.
			\end{align*}
	\end{itemize}
\end{lem}

\thmconvbayes*
\begin{proof}
	For any function $f$, define its regret version $r_{f}$ by 
	\begin{align*}
		r_{f}(x, y) = L(y, f(x)) - L(y, f^{*}(x)).
	\end{align*}
	We call $\setR = \{r_{f}: f \in \setH\}$ the regret class of $\setH$.
	
	For any i.i.d. sample $\bz = ((x_{1}, y_{1}), \ldots, (x_{n}, y_{n}))$, we
	first show that  
	$h(\bz) = \sup_{r \in \setR} |\empproc r - \empproc_{n} r|$ is concentrated
	around its expectation.
	As a first step,	we verify that $h(\bz)$ satisfies the bounded difference
	property with a bound $4 c_{\phi} B/n$, as follows.
	Let $\bz'$ be the same as $\bz$ except that $(x_{i}, y_{i})$ is replaced by an
	arbitrary 
	$(x'_{i}, y'_{i}) \in \setX \times \setY$.
	Since $\phi$ is $c_{\phi}$-Lipschitz and $\phi(0) = 0$, we have
	\begin{align*}
		L(y, f(x)) 
		&= \phi(f(x) - y)  - \phi(0) \\
		&\le c_{\phi} |f(x) - y - 0| \\
		&\le 2 c_{\phi} B,
	\end{align*}
	thus we have 
	\begin{align*}
		&\quad \left|h(\bz) - h(\bz')\right| \\
		&= \left|
			\sup_{r \in \setR} |\empproc r - \empproc_{n} r| 
			- 
			\sup_{r \in \setR} |\empproc r - \empproc_{n}' r| 
			\right|\\
		&\le \sup_{r \in \setR} |\empproc_{n} r - \empproc_{n}' r| \\
		&= \sup_{f \in \setH} \frac{1}{n}
			\bigg|
			\left(L(y_{i}, f(x_{i})) - L(y_{i}, f^{*}(x_{i}))\right) \\
			&\qquad -\left(L(y_{i}', f(x_{i}')) - L(y_{i}, f^{*}(x_{i}))\right)
			\bigg| \\
		&\le \frac{4 c_{\phi} B}{n}.
	\end{align*}
	Hence $h(\bz)$ satisfies the bounded difference property with a bound 
	$4 c_{\phi} B/n$.
	By McDiarmid's inequality, with probability at least $1 - \delta$,
	\begin{align*}
		\sup_{r \in \setR} |\empproc r - \empproc_{n}r| 
		\le \E \sup_{r \in \setR} |\empproc r - \empproc_{n}r|
		+ \frac{\sqrt{8 c_{\phi}^{2} B^{2} \ln(1/\delta)}}{\sqrt{n}}. 
		\tag{I}
	\end{align*}
	By \Cref{lem1}, we have
	\begin{align*}
		\E \sup_{r \in \setR} |\empproc r - \empproc_{n} r| \le 2\E \sup_{r \in \setR} \rproc_{n} r.
		\tag{II}
	\end{align*}

	The Rademacher complexity can be bounded as follows
	\begin{align*}
		&\qquad \E \sup_{r \in \setR} \rproc_{n} r \\
		&= \E \sup_{f \in \conv(\setG)} \rproc_{n} r_{f} \\
		&= \E \sup_{f \in \conv(\setG)} \rproc_{n} (L(\cdot, f(\cdot)) - L(\cdot, f^{*}(\cdot)))\\
		&\le \E \sup_{f \in \conv(\setG)} \rproc_{n} \phi(f(\cdot) - \cdot)
			+ \frac{2c_{\phi}B}{\sqrt{n}} \\
		&\le 2 c_{\phi} \E \sup_{f \in \conv(\setG)} \rproc_{n} (f(\cdot) - \cdot)
			+ \frac{2c_{\phi}B}{\sqrt{n}} \\
		&\le 2 c_{\phi} \E \sup_{f \in \conv(\setG)} \rproc_{n} f
			+ \frac{2B c_{\phi}}{\sqrt{n}} 
			+ \frac{2c_{\phi}B}{\sqrt{n}} \\
		&= 2 c_{\phi} \E \sup_{f \in \conv(\setG)} \rproc_{n} f
			+ \frac{4c_{\phi}B}{\sqrt{n}},
	\end{align*}
	where 1st and 3rd inequalities follow from \Cref{lem2} (b), and
	2nd inequality from \Cref{lem2} (a).
	
	We bound the Rademacher complexity $\E \sup_{f \in G} \rproc_{n} f$ using a
	covering number argument, as follows.
	For any sequence $\bx = (x_{1}, \ldots, x_{n}) \in \setX^{n}$, any 
	$f, g \in \setH$, let
	$d_{\bx}(f, g) = \left(\frac{1}{n} \sum_{i} (f(x_{i}) - g(x_{i}))^{2}\right)^{1/2}$.
	Since $d_{P}(\setG) = p$, for any $\bx$,
	\begin{align}
		N(\epsilon, \setG, d_{\bx}) \le \left(\frac{C}{\epsilon}\right)^{2p},
	\end{align}
	where $C > 0$ is an absolute constant.
	Using Dudley's entropy integral bound, we have 
	\resizebox{.97\linewidth}{!}{\parbox{\linewidth}{%
	\begin{align*}
		\E \sup_{f \in G} \rproc_{n} f 
		\le C'B \E_{\bz} \int_{0}^{\infty} \sqrt{\frac{\ln N(\epsilon, G, d_{\bx})}{n}} d \epsilon
		= D B\sqrt{p/n}. \tag{IV}
	\end{align*}
	}}
	where $C'$ and $D$ are absolute constants.
	
	Combining the inequalities (I)-(IV), we have with probability at least $1 -
	\delta$, for any $r \in \setR$, we have
	\begin{multline*}
		\empproc r
		\le \empproc_{n} r + \sup_{r' \in \setR} |\empproc r' - \empproc_{n} r'|  \\
		\le \empproc_{n} r + 
			\frac{2 c_{\phi} B \left(\sqrt{2 \ln(1/\delta)} + D\sqrt{p} + 2\right)}
			{\sqrt{n}}.
	\end{multline*}
\end{proof}


\thmconvclass*
\noindent\emph{Proof sketch.}
By inspecting the proof of \Cref{thm:convclass}, we can see that the proof still
works if we replace $f^{*}$ by an arbitrary function, including $h^{*}$.
Thus with probability at least $1-\delta$, we have
\begin{align*}
	R(\hat{f}) - R(h^{*})
	\le 
	R_{n}(\hat{f}) - R_{n}(h^{*}) + \frac{c}{\sqrt{n}}
	\le 
	\frac{c}{\sqrt{n}}.
\end{align*}
\hfill$\square$

\corlp*
\begin{proof}
	We have $L(y, f(x)) = \phi(f(x) - y)$ where $\phi(u) = |u|^{q}$.
	$\phi(u)$ has a Lipschitz constant of $q (2B)^{q-1}$ in our case, because 
	we only consider $u$ of the form $f(x) - y$, and thus $u \in [-2B, 2B]$, and
	for any $u, v \in [-2B, 2B]$, we have
	$|\phi(u) - \phi(v)| = |q c^{q-1}  (|u| - |v|)| \le q (2B)^{q-1} |u - v|$,
	where $c$ is between $|u|$ and $|v$.
	The first equation is obtained by applying the mean value theorem to the
	function $u^{p}$, and taking absolute values on both sides.
	The second inequality follows because both $|u|$ and $|v|$ are not more than
	$2B$, and $||u| - |v|| \le |u - v|$.
\end{proof}

\corcls*
\noindent\emph{Proof sketch.}
First note that the loss $L$ is bounded by $c_{\phi} B$ because 
$\phi$ is Lipschitz, $\phi(0) = 0$, and the margin $y f(x)$ is in $[-B, B]$.
Thus $L$ is also bounded by $2 c_{\phi} B$.
This allows the McDiarmid's inequality in the proof of \Cref{thm:convbayes} to go
through.
Secondly, Eq. (*) in the proof of \Cref{thm:convbayes} becomes 
\begin{align*}
	\E \sup_{f \in \conv(\setG)} \rproc_{n} \phi(\cdot f(\cdot)) + \frac{2c_{\phi}B}{\sqrt{n}},
\end{align*}
where the first $\cdot$ is the output $y$, and the second $\cdot$ is the input
$x$.
Now we can apply the Lipschitz property of $\phi$ to bound the first term by
$2 c_{\phi} \E \sup_{f \in \conv(\setG)} \rproc_{n} \cdot f(\cdot)$.
This is equal to 
$2 c_{\phi} \E \sup_{f \in \conv(\setG)} \rproc_{n} f(\cdot)$, because 
$y \in \{-1, 1\}$.
Thus the bound (III) still holds.
The remaining steps go through without any modification.
\hfill$\square$

\corlog*
\begin{proof}
	Consider the loss 
	$L'(y, f(x) = -\ln \frac{2}{1 + e^{-y f(x)}}$.
	Then $L'(y, f(x)) =  L(y, f(x)) - \ln 2$, that is, $L$ and $L'$ differ by only
	a constant, and thus it is sufficient to show that the bound holds for $L'$.
	The modified loss $L'(y, f(x))$ has the form $\phi(y f(x))$ where 
	$\phi(u) = -\ln \frac{2}{1 + e^{-u}}$.
	We have $\phi(0) = 0$.
	In addition, $\phi$ is 1-Lipschitz because the absolute value of its
	derivative is $|\phi'(u)|  = |\frac{-e^{-u}}{1 + e^{-u}}| \le 1$.
	Hence $L'$ satisfies the condition in \Cref{cor:cls}, and the $O(1/\sqrt{n})$
	bound there holds for $L'$.
\end{proof}

\section{Datasets and Experimental Setup}\label{apdx:experiments}

We used 12 datasets of various sizes and tasks (both regression and classification) in the experiments. Most of the datasets are from UCI ML Repository \citep{uci}. Table \ref{tbl:datasets} provides a summary of the datasets. 

\begin{table}[h!]
\resizebox{.7\linewidth}{!}{\parbox{\linewidth}{%
\begin{tabular}[t]{llrrrrrr} 
\toprule
\textbf{datasets} & \textbf{task} & \#\textbf{dim} & \#\textbf{samples} & \#\textbf{train} & \#\textbf{valdtn} & \#\textbf{test} \\
\midrule
diabetes & regr & 10 & 442 & 282 & 71 & 89 \\
boston & regr & 13 & 506 & 323 & 81 & 102 \\
ca\_housing & regr & 8 & 20,640 & 13,209 & 3,303 & 4,128 \\
msd & regr & 90 & 515,345 & $^\dagger$296,777 & 74,195 & $^\dagger$144,373 \\
iris & class & 4 & 150 & 96 & 24 & 30 \\
wine & class & 13 & 178 & 113 & 29 & 36 \\
breast\_cancer & class & 30 & 569 & 364 & 91 & 114 \\
digits & class & 64 & 1,797 & 1,149 & 288 & 360 \\
cifar10\_f & class & 342 & 60,000 & $^\dagger$32,000 & 8,000 & $^\dagger$20,000 \\
mnist & class & 784 & 70,000 & $^\dagger$38,400 & 9,600 & $^\dagger$22,000 \\
covertype & class & 54 & 581,012 & $^\dagger$11,340 & $^\dagger$2,268& $^\dagger$569,672 \\
kddcup99 & class & 41 & 4,898,431 & $^\dagger$2,645,152 & 661,288 & $^\dagger$1,591,991 \\
\bottomrule
\end{tabular}
}}
\caption{\label{tbl:datasets} \emph{Summary of datasets used in the experiments.} $^\ddagger$: number of classes for classification and $\max |output|$ for regression. $^\dagger$: a predefined set chosen by the author of the respective dataset. If not predefined, training/test set is a random split of 80\%/20\% of the dataset. Then, the training set is again split into 80\% for training and 20\% for validation. \emph{boston, diabetes, iris, digits, wine, breast\_cancer, covertype, kddcup99, ca\_housing} are loaded using data utility in scikit-learn package. (Year prediction) \emph{msd} is taken directly from UCI's website. \emph{mnist} and \emph{cifar10} are well known datasets and are loaded using PyTorch data utilities. cifar10\_f is created by transforming cifar10 using a trained DenseNet (95.16\% accuracy on test set) and taking the values of 342 features in the last convolutional layer. Training/validation/test set remain identical across all experiments. All data attributes are normalized to have 0-mean 1-standard deviation. }
\end{table}

To ensure a fair comparison between algorithms, we spent a significant effort to tune the hyper-parameters of competing algorithms. For each algorithm, the model with the best validation performance is selected and the performance on test set is reported. Small datasets (less than 10000 training samples, i.e. \emph{boston, diabetes, iris, digits, wine, breast\_cancer}), allow for a larger number of hyper-parameter tuning combinations. The details of hyper-parameter tuning for each algorithm is as follows:

\begin{itemize}

	\item XGBoost: Evaluation metric \textit{eval\_metric=merror} for classification and \textit{eval\_metric=rmse} for regression. \textit{num\_boost\_round=2000, early\_stopping\_rounds=50}. For small datasets, we start with \textit{eta=0.1} and do 100 random searches over the two most important hyper-parameters in the following ranges: \textit{max\_depth} $\in \{3, 5, 7, 9, 11, 13\}$, \textit{min\_child\_weight} $\in \{0, 2, 4, 6, 8, 10, 12, 14, 16, 18\}$. The random search is followed by a $5\times5$ fine tuning grid search around the best value for each parameters. Next, we tune \textit{gamma} $\in \{0.0, 0.1, 0.2, 0.3, 0.4, 0.5\}$, followed by a grid search for \textit{subsample} $\in \{0.5, 0.6, 0.7, 0.8, 0.9\}$ and \textit{colsample\_bytree} $\in \{0.5, 0.6, 0.7, 0.8, 0.9\}$, followed by another grid search for \textit{reg\_lambda} $\in \{0, 0.0001, 0.001, 0.01, 0.1, 1, 10, 100\}$ and \textit{reg\_alpha} $\in \{0, 0.1, 1\}$. Next, we tune the learning rate \textit{eta} $\in \{0.2, 0.15, 0.05, 0.01, 0.005, 0.001\}$. Then finally we do a 1000 random searches in the neighborhood of the best value for all parameters. This process generates in total about 2298 combinations of hyper-parameters settings. For large datasets, the procedure is similar, with more restrictive range of values for secondary parameters: \textit{gamma} $\in \{0, 0.05, 0.1, 0.2\}$, \textit{subsample} $\in \{0.5, 0.6, 0.7, 0.8\}$, \textit{colsample\_bytree} $\in \{0.5, 0.6, 0.7, 0.8\}$, \textit{reg\_lambda} $\in \{0, 0.5, 1, 1.5\}$, \textit{eta} $\in \{0.15, 0.05, 0.01\}$, which are also optimized separately instead of jointly in pairs as before. 

	\item RandomForest: Training criterion is $gini$ for classification and $mse$ for regression. The maximum number of trees is 2000, with early stoping if there is no improvement over 50 additional trees. For Random Forest, we found that a large number of random searches is often the most effective strategy. So, we do 4000/200 random searches for small/large datasets respectively. The hyper-parameters and their range are as follows.\\ 
\textit{max\_features} $\in \{auto, sqrt, log2, 1, 3, 5, 0.1, 0.2, 0.3, \\0.4, 0.5, 0.6, 0.7\}$,\\ \textit{min\_samples\_leaf} $\in \{1, 2, 5, 10, 20, 30, 50, 80, 120, 170,\\230\}$,\\ 
\textit{max\_depth} $\in \{1, 2, 5, 10, 20, 30, 50, 80, 120\}$, \\ 
\textit{min\_samples\_split} $\in \{2, 4, 6, 8, 10, 12, 14, 16\}$,\\ 
\textit{no\_bootstrap} $\in \{True, False \}$.

	\item GCE, NGCE, NN: The basis module for GCE is a two layers network with 1 or 10 hidden neurons for small datasets  and 100 hidden neurons for other datasets. The output of the basis module is bounded using the $hardtanh$ function scaled by the bound $B$. For classification we set $B=10$, for regression, $B=\frac{4}{3} \max |output|$. GCE grows the network by adding one basis module at a time until no improvement on validation set is detected, or until reaching the maximum limit of 100 modules. We use Brent's method as line search for parameter $\alpha_t$ of GCE. NGCE and NN are given the maximum capacity achievable by GCE. For NGCE, it is 100 modules, each of 10/100 hidden neurons for small/large datasets. NN is a two layers neural net of 1000/10000 hidden neurons for small/large datasets, respectively. For both small and large datasets, NGCE and NN are tuned using grid search over $learning\_rate \in \{0.01, 0.001\}$, $regularization \in \{0, 10^{-6}, \dots, 10^{-1} \}$, totaling in 14 combinations. GCE uses a fixed set of hyper-parameters without tuning: $learning\_rate=0.001, regularization=0$. All these three algorithms use ReLU activation, MSE criterion for training regression problem, cross entropy loss for classification. The training uses Adam SGD with $learning\_rate$ reduced by 10 on plateau (training performance did not improve for 10 consecutive epochs) until reaching the minimum learning rate of $10^{-5}$, at which point the optimizer is ran for another 10 epochs and then returns the  solution with the best performance across all training epochs.

\end{itemize}

All experiments are implemented using Python and its interface for XGBoost.  Random Forest is from scikit-learn package. Greedy variants and NN are implemented using PyTorch \citep{paszke2017automatic} and are ran on a machine with Intel i5-7600K CPU @ 3.80GHz (4 cores) and 1x NVIDIA GEFORCE GTX 1080 Ti GPU card. XGBoost and Random Forest are run on cloud machines with Intel CPU E5-2650 v3 @ 2.30GHz (8  cores) and one NVIDIA GEFORCE RTX 2080 Ti GPU.
}
{}

\end{document}